\documentclass[10pt,journal,compsoc]{IEEEtran}

\usepackage{amsthm}
\usepackage{amsmath}
\usepackage{amsfonts}
\usepackage{array}
\usepackage{graphicx}
\usepackage{subcaption}
\usepackage{multirow}
\usepackage{color}
\usepackage{hyperref}
\hypersetup{
    colorlinks=true,
    linkcolor=blue,
    filecolor=magenta,      
    urlcolor=cyan,
    pdftitle={Overleaf Example},
    pdfpagemode=FullScreen,
    }
\usepackage{bbm}
\usepackage[mathscr]{eucal}
\newcommand{\tabincell}[2]{\begin{tabular}{@{}#1@{}}#2\end{tabular}}  
\newcommand{\feadim}{D}
\newcommand{\ngs}{P}
\newtheorem{theorem}{Theorem}
\newtheorem{definition}{Definition}
\newcommand{\mycolor}{\textcolor{black}}

\begin{document}

\title{Towards Improved and Interpretable Deep Metric Learning via Attentive Grouping}

\author{Xinyi Xu,
        Zhengyang Wang,
        Cheng Deng*,~\IEEEmembership{Senior Member,~IEEE},
        Hao Yuan,
        and Shuiwang Ji*,~\IEEEmembership{Senior Member,~IEEE}
        \IEEEcompsocitemizethanks{\IEEEcompsocthanksitem X. Xu and C. Deng are with the School of Electronic Engineering, Xidian University, Xi'an 710071, China. E-mail: \{xyxu.xd, chdeng.xd\}@gmail.com \protect\\
        \IEEEcompsocthanksitem Z. Wang, H. Yuan, and S. Ji are with the Department of Computer Science and Engineering, Texas A\&M University, College Station, TX 77843. E-mail: \{zhengyang.wang, hao.yuan, sji\}@tamu.edu}
        \thanks{This work was performed while the first author is visiting Texas A\&M University. *Correspondence should be addressed to these authors.}
}

\IEEEtitleabstractindextext{

\begin{abstract}
Grouping has been commonly used in deep metric learning for computing diverse features. To improve the performance and interpretability, we propose an improved and interpretable grouping method to be integrated flexibly with any metric learning framework. Our method is based on the attention mechanism with a learnable query for each group. The query is fully trainable and can capture group-specific information when combined with the diversity loss. An appealing property of our method is that it naturally lends itself interpretability. The attention scores between the learnable query and each spatial position can be interpreted as the importance of that position. We formally show that our proposed grouping method is invariant to spatial permutations of features. When used as a module in convolutional neural networks, our method leads to translational invariance. We conduct comprehensive experiments to evaluate our method. Our quantitative results indicate that the proposed method outperforms prior methods consistently and significantly across different datasets, evaluation metrics, base models, and loss functions. For the first time to the best of our knowledge, our interpretation results clearly demonstrate that the proposed method enables the learning of distinct and diverse features across groups. The code is available on \url{https://github.com/XinyiXuXD/DGML-master}.
\end{abstract}

}

\maketitle
\IEEEdisplaynontitleabstractindextext

\IEEEpeerreviewmaketitle

\IEEEraisesectionheading{\section{Introduction}}
Deep metric learning (DML) computes image representations by employing a deep neural network (DNN) \cite{lecun1998gradient,Ji:TPAMI2012,Gao:TPAMI19, zhengyang2020pami} to map images from the original pixel space to a feature embedding space. 
The learned representations have been widely applied in downstream computer vision tasks, including image clustering \cite{oh2016deep}, image retrieval \cite{sohn2016improved, wang2019multi}, face verification \cite{schroff2015facenet}, etc. 
\mycolor{DML has achieved significant progress by
designing different metric loss functions \cite{hadsell2006dimensionality, schroff2015facenet, sohn2016improved, wu2017sampling} to guide the training procedure of DNN. The metric loss function is crucial for DML models since it provides guidance on network training and encourages the learned representations to capture discriminative information. Commonly used metric loss functions include 
the contrastive loss \cite{hadsell2006dimensionality}, the triplet loss \cite{schroff2015facenet}, the n-pair loss \cite{sohn2016improved}, and the margin loss \cite{wu2017sampling}. In addition, several strategies have been proposed to facilitate the training of DML models, such as lifted triplet \cite{oh2016deep}, n-pair \cite{sohn2016improved}, distance mining \cite{wu2017sampling}, and hardness aware \cite{yuan2017hard, zheng2019hardness}. } 

\mycolor{While DNNs in DML are powerful, traditional DML methods only learn a vector in a single embedding space for the whole image, which may not fully capture the semantics of inputs. This drawback causes a dimension saturation issue, which is observed when the improvements of performance saturate with the increasing of feature dimensions~\cite{oh2016deep,opitz2018deep,roth2020revisiting}.
To overcome this issue, recent studies \cite{opitz2018deep, kim2018attention, chen2019hybrid, xuan2018deep, aziere2019ensemble} split the large embedding dimension into different groups and learn multi-embedding representations. Specifically, multiple feature vectors are learned and jointly represent the input image. Each feature vector corresponds to one group and different groups are expected to capture different input characteristics. To be apart from traditional DML techniques, we categorize such methods as   
deep grouping metric learning (DGML). Notably, DGML methods generally outperform traditional single-embedding DML methods since grouping can encourage the output embeddings to capture more comprehensive characteristics~\cite{opitz2018deep, milbich2020diva}. However, current DGML methods\cite{opitz2018deep,kim2018attention,chen2019hybrid, xuan2018deep, aziere2019ensemble} cannot ensure different groups capturing different and discriminative characteristics, which limits their performance and interpretability.}

\mycolor{In this work, we propose an improved and interpretable DGML method, named  A-grouping. Our A-grouping module
splits the entire embedding spaces into different groups
by using different position importance for each group. Then position-wise importance scores are obtained via attention mechanism \cite{vaswani2017attention}. Attention operations compute the outputs based on the multiplications between the query, key, and value matrices.
Different from self-attention~\cite{vaswani2017attention}, in which these matrices are obtained from the input features via linear transformations, our A-grouping
computes a learnable query vector for each group and obtain the importance scores by attending the query to spatial positions of
the key matrix.}
The query is fully trainable and can capture the information distributed on these important positions when combined with the metric loss. In addition, it can capture group-specific information when combined with the diversity loss. An appealing property of our method is that it naturally lends itself interpretability. The attention scores between the learnable query and each spatial position can be interpreted as the importance of that position. We formally show that our proposed grouping method is invariant to spatial permutations of features. When used as a module in convolutional neural networks, our method leads to translational invariance. This property enables our method to compute representations that are invariant to translations on the original input images. 
We conduct comprehensive experiments to evaluate our method. Our quantitative results indicate that the proposed method outperforms prior methods consistently and significantly across different datasets, evaluation metrics, base models, and loss functions. For the first time to the best of our knowledge, our interpretation results clearly demonstrate that the proposed method enables the learning of distinct and diverse features across groups. 

\section{Background and Related Work}
In this section, we first give the formal problem formulation of the deep metric learning (DML) and introduce important loss functions in DML. Then we discuss the deep grouping metric learning (DGML) and related studies.

\subsection{Problem Formulation of DML}
DML is commonly studied under the zero-shot learning settings. To be specific, the testing data are composed of images from classes that are not included in the training data. Formally, let $\{\boldsymbol{X}_i, y_i\}_{i=1}^{N}$ and $\{\hat{\boldsymbol{X}}_i, \hat{y}_i\}_{i=1}^{M}$ denote the training and testing data, respectively. Here, $\boldsymbol{X}_i, \hat{\boldsymbol{X}}_i \in \boldsymbol{\mathcal{X}}$ represent images and $y_i, \hat{y}_i \in \mathbb{N}$ are the corresponding labels. In zero-shot learning, the training and testing data have disjoint label sets, \emph{i.e.}, $\{y_i\}_{i=1}^{N} \cap \{\hat{y}_i\}_{i=1}^{M} = \emptyset$.

During training, DML learns an embedding mapping $f: \boldsymbol{\mathcal{X}} \mapsto \boldsymbol{\mathcal{F}} \subseteq \mathbb{R}^{\feadim}$ with a deep learning model, which maps images to feature vectors in the $\feadim$-dimensional embedding space $\boldsymbol{\mathcal{F}}$. With $f$, the similarity between any two images $\boldsymbol{X}_i$ and $\boldsymbol{X}_j$ can be measured based on $\boldsymbol{f}_i = f(\boldsymbol{X}_i)$ and $\boldsymbol{f}_j = f(\boldsymbol{X}_j)$. For example, a small Euclidean distance $d(\boldsymbol{f}_i, \boldsymbol{f}_j) = \|\boldsymbol{f}_i-\boldsymbol{f}_j\|_{2}$ or a large cosine similarity $s(\boldsymbol{f}_i, \boldsymbol{f}_j) = \boldsymbol{f}_i^{T}\boldsymbol{f}_j/(\|\boldsymbol{f}_i\|_{2}^{2} \cdot \|\boldsymbol{f}_j\|_{2}^{2})$ both can indicate that $\boldsymbol{X}_i$ and $\boldsymbol{X}_j$ are similar. In DML, we consider images belonging to the same class to be similar to each other, while images with different labels have low similarities. As a result, the learning objective of DML is to train the embedding mapping $f$ such that the similarities between images are correctly reflected in the corresponding embedding space $\boldsymbol{\mathcal{F}}$ with respect to a predefined metric. With this objective, various loss functions have been proposed for DML.

In practice, $f$ is usually implemented by a backbone CNN for image classification, whose last classification layer is replaced by an output module producing $\feadim$-dimensional feature vectors. The CNN is typically pretrained on a large-scale image classification dataset and then fine-tuned on the DML training dataset $\{\boldsymbol{X}_i, y_i\}_{i=1}^{N}$ with the replacement. After training, the quality of $f$ is evaluated through image retrieval and image clustering tasks \cite{sohn2016improved} on the testing data $\{\hat{\boldsymbol{X}}_i, \hat{y}_i\}_{i=1}^{M}$. Specifically, we perform these tasks based on $f(\hat{\boldsymbol{X}}_i)$ and use the corresponding labels $\hat{y}_i$ to measure the performance.

\subsection{\mycolor{Three Representative DML Loss Functions }}

A major direction in DML studies is the development of various loss functions. In the following, we introduce three representative loss functions in DML.
The contrastive loss \cite{chopra2005learning} 
minimizes the Euclidean distance of positive pairs (images with the same label), while pushing that of negative pairs (images with different labels) above a preset margin. Formally, given $(\boldsymbol{X}_i, y_i)$ and $(\boldsymbol{X}_j, y_j)$ the contrastive loss is formulated as
\begin{equation}
   \mathcal{L}_{\text{con}} = l_{i,j}*d_{i,j} + (1-l_{i,j})*\left[m-d_{i,j}\right]_+, 
   \label{con_loss}
\end{equation}
where $d_{i,j} = d(\boldsymbol{f}_i, \boldsymbol{f}_j)$ is the Euclidean distance between $\boldsymbol{f}_i = f(\boldsymbol{X}_i)$ and $\boldsymbol{f}_j = f(\boldsymbol{X}_j)$, $l_{i,j} = \mathbbm{1}_{\{y_{i}=y_{j}\}}$ is the pair indicator that is equal to $1$ for positive pairs and $0$ for negative pairs, $m$ is the preset margin and $[\cdot]_+$ represents the hinge loss function. 

The binomial deviance loss \cite{yi2014deep} employs the cosine similarity and applies the preset margin for both positive pairs and negative pairs. Specifically, it is defined as
\begin{equation}
\begin{split}
 \mathcal{L}_{\text{bin}}= &l_{i,j} * \log\left(1+e^{-\alpha(s_{i,j}-m)\beta_{1}}\right) + \\
     &(1-l_{i,j}) * \log\left(1+e^{\alpha(s_{i,j}-m)\beta_{0}}\right),
\end{split}
\label{bd_loss}
\end{equation}
where $s_{i,j} = s(\boldsymbol{f}_i, \boldsymbol{f}_j)$ computes the cosine similarity, $m$ is the preset margin, and $\alpha$, $\beta_{0}$ and $\beta_{1}$ are scaling factors.

\mycolor{In addition, the margin loss \cite{wu2017sampling} uses the Euclidean distance to measure the similarity among different example pairs and applies the preset
margin for all example pairs.} In particular, it uses different preset margins for positive pairs and negative pairs with two hyperparameters $\eta$ and $m$. Concretely, the margin loss is given by
\begin{equation}
\begin{split}
    \mathcal{L}_{\text{mar}} = &l_{i,j}  * \left[d_{i,j}-(\eta-m)\right]_{+} + \\
    &(1-l_{i,j}) * \left[(\eta + m)-d_{i,j}\right]_{+}.
\end{split}
\label{mar_loss}
\end{equation}

Next, we discuss DGML and focus on grouping methods. It is worth noting that grouping methods can generally be used together with different loss functions.

\begin{figure*}[t]
	\centering
	\includegraphics[width=.98\linewidth]{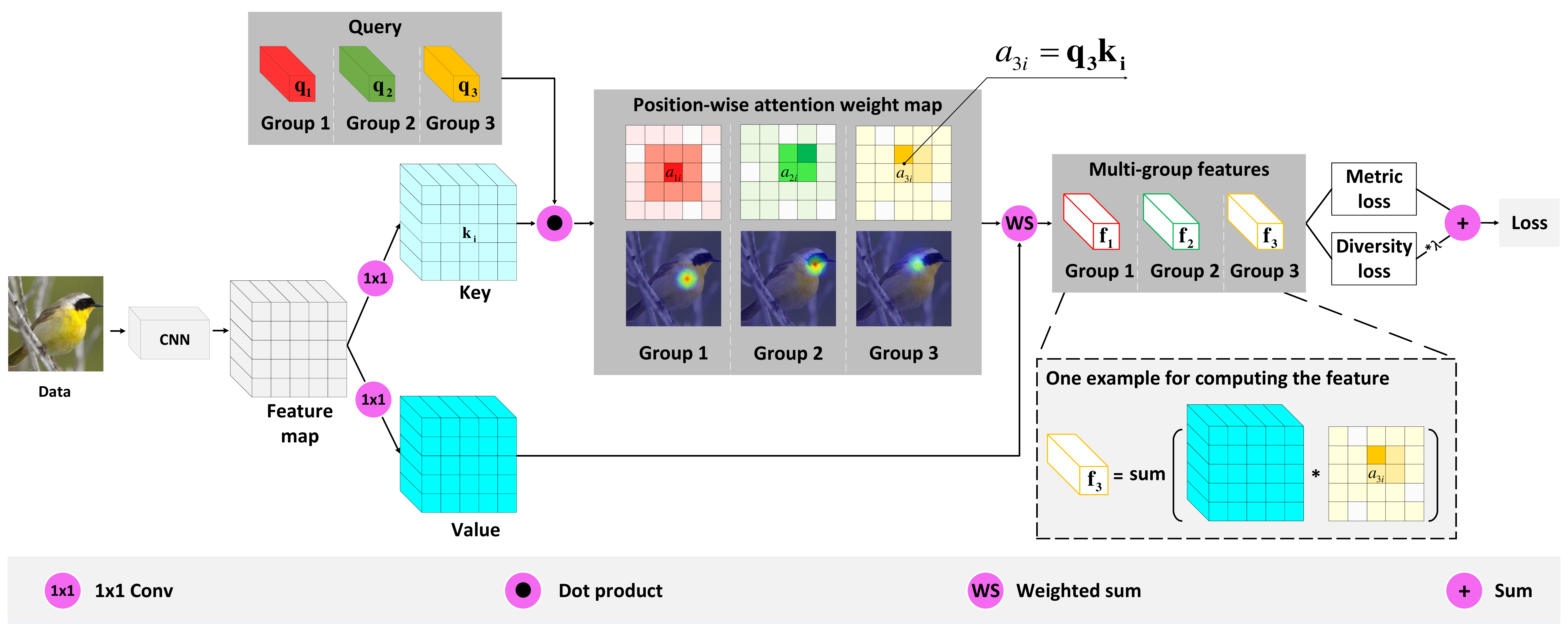}
	\caption{An illustration of our proposed deep metric learning framework. Our A-grouping module takes the feature maps generated by CNNs as the input and outputs multi-group features. In this example, the number of groups is set to three. The key and value tensors (blue cube) are first computed by two independent linear transformations of the feature maps, which is implemented by the $1 \times 1$ convolution. Then the position-wise weight maps are generated for each group by attending the corresponding query vector to the group-shared key. 
	Finally, we compute the weighted summation over the  
	value tensor with group-specific weight maps to obtain the output feature vectors for different groups. We provide an example in the dashed box to explain the computation procedure for group 3. 
    To train the model, the metric loss is applied to each group separately, whereas the diversity loss is computed across every two groups.}
	\label{framework}
\end{figure*}

\subsection{Deep Grouping Metric Learning}
\mycolor{Recently, deep grouping metric learning (DGML) has achieved remarkable progress, which learns multiple embedding vectors to replace the single embedding vector in DML. 
It is motivated by the observation that DML suffers from the dimension saturation issue~\cite{oh2016deep, opitz2018deep, roth2020revisiting}.} Specifically, the performance of DML increases when the dimension of the embedding space $\feadim$ becomes larger, due to the increased model capacity. However, the amount of improvement gets smaller and smaller with increasing $\feadim$. In certain cases, increasing $\feadim$ even hurts the performance.

The causes of the dimension saturation issue can be analyzed in two aspects. First, with a large $\feadim$, overfitting is expected \cite{opitz2018deep, chen2019hybrid}, as extra dimensions of the embedding space $\boldsymbol{\mathcal{F}}$ may capture noises in the training data and lead to poor generalization. Second, we hypothesize that projecting images into a single embedding space $\boldsymbol{\mathcal{F}}$ may cause all the dimensions to statistic the characteristics of input images under one distribution only.

Based on these insights, a straightforward solution to the dimension saturation issue is to have multiple feature vectors providing different \mycolor{characteristics} of the input images. Moreover, each feature vector has a smaller dimension in order to avoid overfitting. Specifically, the original output module in DML producing a $\feadim$-dimensional feature vector is replaced a grouping module generating $\ngs$ feature vectors with smaller dimensions, where the sum of dimensions is still $\feadim$. As a result, the single $\feadim$-dimensional embedding space $\boldsymbol{\mathcal{F}}$ is decomposed into a concatenation of multiple compact embedding spaces $\boldsymbol{\mathcal{F}}_i$, $i=1,2,\ldots,\ngs$. During training, the  objective is applied on each $\boldsymbol{\mathcal{F}}_i$ independently and a diversity loss is added to force each $\boldsymbol{\mathcal{F}}_i$ to encode distinct and diverse features of input images. \mycolor{To distinguish from traditional DML methods, we categorize such DML methods with grouping modules as DGML \cite{opitz2018deep, kim2018attention, chen2019hybrid, xuan2018deep, aziere2019ensemble}}. Ideally, DGML can address the dimension saturation issue and yield further improvements with a large $\feadim$.

\mycolor{In DGML, the key component is the design of different grouping modules~\cite{opitz2018deep, kim2018attention, chen2019hybrid}. Recently, BIER \cite{opitz2018deep} proposes to project the input into different feature spaces by employing multiple linear projection functions, which can be named M-grouping. It employs the sample reweighting strategy to encourage different groups to focus on different data.  Instead of using linear projection functions, later efforts have been devoted to designing nonlinear grouping modules \cite{kim2018attention, chen2019hybrid}. For example, recent work~\cite{kim2018attention} proposes a grouping module based on gated convolution networks~\cite{yu2019free} and we name it as G-grouping module. It learns an element-wise weight mask for each group from the 
input feature maps. In addition, C-grouping~\cite{chen2019hybrid} is recently proposed to learn channel-wise weights for the input feature maps. Specifically, it employs the channel-wise attention~\cite{hu2018squeeze} to perform the non-linear transformations.
Generally, these DGML methods~\cite{opitz2018deep, kim2018attention, chen2019hybrid} follow the same pipeline which consists of a shared backbone and a grouping module.}


\mycolor{While DGML methods are more powerful than traditional DML methods, it is not clear whether the different groups in DGML indeed capture diverse characteristics, which may limit the performance and interpretability.  } In this work, we propose an novel DGML model, in the core of which lies a powerful and interpretable attentive grouping (A-grouping) module. 
Our A-grouping module achieves consistently and significantly improved performances across different datasets, evaluation metrics, base models, and loss functions. 
More importantly, the proposed A-grouping module yields meaningful interpretation results, which clearly demonstrates that the distinct and diverse features across groups are learned.

\section{The Proposed Methods}
In this work, we propose a novel attentive grouping module, \mycolor{named} A-grouping, to improve the performance and interpretability of DGML models. Specifically, our A-grouping module employs the attention mechanism to learn group-aware embeddings for the output feature maps of the CNNs. As shown in Figure~\ref{framework}, our non-linear A-grouping module first learns distinct queries for different groups and uses these queries to obtain position-wise attention weight maps. By employing the attention mechanism, our A-grouping module can effectively incorporate global information from the input to produce the weight maps.
In addition, our proposed A-grouping module can build interpretable DGML models. Since each attention map corresponds to one group, then we can understand the correspondence between input pixels and feature groups by visualizing attention maps. In this section, we first present the problem formulation of the DGML problem in Section~\ref{sec:formualtion}. Next, we introduce our A-grouping module and its interpretability in Section~\ref{sec_a_grouping} and~\ref{sec:interpretability}. Finally, we discuss the permutation invariance and translation invariance properties of our proposed A-grouping module in Section~\ref{sec:invariance}.

\subsection{Problem Formulation of DGML}~\label{sec:formualtion}
We first formally define the DGML problem. Given an input image, the original DML problem learns its feature vector in the embedding space. Different from DML, the goal of the DGML problem is to map the input image to multiple embedding spaces (groups).
Specifically, suppose each feature vector contains $\ngs$ groups. Then the original embedding mapping $f: \boldsymbol{\mathcal{X}} \mapsto \boldsymbol{\mathcal{F}} \subseteq \mathbb{R}^{\feadim}$ is replaced by $\ngs$ embedding mappings $f_{i}: \boldsymbol{\mathcal{X}} \mapsto \boldsymbol{\mathcal{F}}_{i} \subseteq \mathbb{R}^{\feadim_{i}}$, $i=1,2,\ldots,\ngs$, where $\feadim_1 + \feadim_2 + \cdots + \feadim_{\ngs} = \feadim$. In each feature space, the similarity and label consistent constraints should be satisfied. That is, the metric learning loss is independently computed and optimized for each space by Eqs.~(\ref{con_loss}), (\ref{bd_loss}), or (\ref{mar_loss}). Simultaneously, the diversity loss is employed to encourage the diversity between every two groups, where the formula is
\begin{equation}\label{l_d}
	\mathcal{L}_{\text{div}} = \log\left(1+e^{\alpha(s_{i,j}-\mu)\beta_{0}}\right),
\end{equation}
where $s_{i,j}$ is the cosine distance of features from two groups of the same image.
The diversity loss is in the form of binomial deviance loss, where only negative cases are considered. 

Commonly, the embedding mapping of DGML is composed of one group-shared feature extractor which is a CNN backbone, and $\ngs$ group-specific grouping mappings. Therefore, the problem is converted to how to obtain the grouping mappings. In this work, we propose the A-grouping module to perform group mappings with interpretability based on the attention mechanism.

\subsection{The Proposed Attentive Grouping Method} \label{sec_a_grouping}
The attention mechanism was originally developed in the natural
language processing domain~\cite{vaswani2017attention, LiuTextICDM19} and has been
extended to deal with image and video
data~\cite{Gao:KAN,zwangUnetAAAI20,wang2018non,Liu:TMI2020}. In this work, we propose to
develop a novel grouping method based on the attention mechanism, \mycolor{called} A-grouping. We
show that our grouping method can not only lead to improved metric
learning performance, but also endows interpretability to the
resulting model.

Our proposed A-grouping takes a set of
feature maps as the input and generates multiple groups of feature
vectors as the output. In particular, let $\boldsymbol{\mathscr{I}} \in
\mathbb{R}^{H \times W \times C}$ denotes the input tensor
containing feature maps, where $H$, $W$, and $C$ denote the height,
width, and the number of channels, respectively. The input tensor
$\boldsymbol{\mathscr{I}}$ is first processed by two independent
$1\times 1$ convolutions to generate the key and value tensors
$\boldsymbol{\mathscr{K}}$ and $\boldsymbol{\mathscr{V}}$ as
follows:
\begin{eqnarray*}
\boldsymbol{\mathscr{K}} &=& \mbox{Convolution}_{1\times
1}(\boldsymbol{\mathscr{I}})\in \mathbb{R}^{H
\times W \times \feadim_{K}},\\
\boldsymbol{\mathscr{V}} &=& \mbox{Convolution}_{1\times
1}(\boldsymbol{\mathscr{I}})\in \mathbb{R}^{H \times W \times \feadim_{V}},
\end{eqnarray*}
where $\feadim_{K}$ and $\feadim_{V}$ denote the numbers of feature maps in the key and
value tensors, respectively. The key and value tensors are then
unfolded into matrices along mode-3~\cite{kolda2009tensor},
resulting in the key and value matrices as
\begin{eqnarray}
\boldsymbol{K}&=&\mbox{Unfold}_3(\boldsymbol{\mathscr{K}})\in\mathbb{R}^{\feadim_{K}\times HW},\label{eq:unfold1}\\
\boldsymbol{V}&=&\mbox{Unfold}_3(\boldsymbol{\mathscr{V}})\in\mathbb{R}^{\feadim_{V}\times
HW}.\label{eq:unfold2}
\end{eqnarray}
Here the columns of these matrices are the mode-3 fibers of the corresponding
tensors~\cite{kolda2009tensor}.

To perform attentive operations, we next define the query vectors used in our A-grouping. 
Our proposed method aims at using different groups to identify different spatial positions, thereby facilitating the
interpretability of our metric learning model. Hence, we propose to achieve
this by computing a group-specific weight for each spatial position.
To this end, we introduce a set of learnable query vectors $\boldsymbol Q =
[\boldsymbol{q}_{1}, \boldsymbol{q}_{2}, \cdots, \boldsymbol{q}_{\ngs}]\in \mathbb{R}^{\feadim_{K} \times \ngs}$, where $\ngs$
denotes the number of groups. Note that there is one query vector
for each group, and the number of query vectors determines the
number of groups. These query vectors are randomly initialized and
their values are trained along with other parameters in the network~\cite{yang2016hierarchical,LiuTextICDM19,YuanBigNeuronICDM19}.
In addition, the dimensionality of the query vectors needs to be the
same as that of the key vectors as columns of the matrix
$\boldsymbol{K}$. 

To be concrete, let us focus on one particular query $\boldsymbol{q}_{i}$. We
measure the similarity between $\boldsymbol{q}_{i}$ and the key vectors as
columns of the matrix $\boldsymbol{K}$. In this work, we use the
inner product as the similarity measurement. We also normalize the
similarities corresponding to each query using the softmax function
so that similarity values are between 0 and 1 and sum to one for
each query vector. Mathematically, these operations for all query
vectors can be expressed in matrix form as
\begin{equation}\label{eq:S}
\boldsymbol{A} = \mbox{Softmax}(\boldsymbol{Q}^T
\boldsymbol{K})\in\mathbb{R}^{\ngs \times HW},
\end{equation}
where $\boldsymbol{A}$ denotes the similarity score matrix. Note
that the Softmax function is applied to each row of
its input independently. 

Intuitively, the values of the
similarity score matrix $\boldsymbol{A}$ measures the importance of spatial positions. Hence, we use $\boldsymbol{A}$ to generate groups of features as a weighted sum of the value vectors in $\boldsymbol{V}$. Specifically, each row of $\boldsymbol{A}$ represents the position weights for one group. For a specific group $i$, its weights are defined as $\boldsymbol{a}_{i}\in\mathbb{R}^{1 \times HW}$, which is the $i$-th row of $\boldsymbol{A}$.
Note that $\boldsymbol{a}_{i}$ contains $HW$ elements and each element corresponds to one spatial position.
In addition, each column of the value matrix $\boldsymbol{V}$ is corresponding to one spatial position. 
Hence, we multiply each column vector of $\boldsymbol{V}$ by the corresponding element of $\boldsymbol{a}_{i}$ and then sum $HW$ new position vectors to generate the feature vector of $i$-th group. Mathematically, the procedure can be expressed as
\begin{equation}
    \boldsymbol{F} = \boldsymbol{A}\boldsymbol{V}^T\in\mathbb{R}^{\ngs \times \feadim_{V}},
    \label{eq10}
\end{equation}
where $\boldsymbol{F}$ contains the $P$ groups of features as its rows, and the dimension of each group is commonly set to the same, \emph{i.e.}, $\feadim_{V}$. 
We illustrate our proposed A-grouping module in Figure~\ref{framework} where the number of groups is set to $\ngs=3$. Note that the figure is shown in a $2$-D spatial view for intuitive observations.

Finally, after obtaining $\boldsymbol{F}$, we independently compute the metric loss for each group, e.g., as in Eqs.~\eqref{con_loss}, \eqref{bd_loss}, and \eqref{mar_loss}. 
Intuitively, the metric loss will encourage the learnable queries to assign large weights to the more important positions while small weights to the less informative positions. Note that the definition of importance is group-specific and varies for different groups. In addition, the diversity loss is computed across every two different groups by Equation~\eqref{l_d}. This loss encourages these different groups of features to capture different characteristics of the images. Overall, the whole model is trained by using these two losses jointly.

\subsection{Interpretability}\label{sec:interpretability}
Intuitively, with the diversity loss, different groups tend to capture different patterns from input images. However, without investigating what different groups are detecting, it is unknown whether the DGML models work in our expected way. The lack of interpretability may prevent the use of DGML models in critical applications, which is a common limitation of deep learning approaches~\cite{Yuan:TPAMI20, Yuan:AAAI19}.
In this work, our proposed A-grouping employs the attention mechanism, thus making our DGML models interpretable. Specifically, we study the position-wise attention weight maps for different groups. To enable the interpretability of our proposed method, we fold the attention weight matrix $\boldsymbol{A}$ into a 3-way tensor as
\begin{equation}\label{eq:S2T}
\boldsymbol{\mathscr{A}} = \mbox{Fold}(\boldsymbol{A}) \in\mathbb{R}^{H\times W\times \ngs},
\end{equation}
where the fold operation reverses the unfolding operations in Eqs~(\ref{eq:unfold1}) and (\ref{eq:unfold2}). Next, we map the attention weights to the input space. Formally, $\boldsymbol{\mathscr{A}}$ is resized to the same sizes as the input images by the bi-linear interpolation that
\begin{equation}
    \hat{\boldsymbol{\mathscr{A}}} = \mbox{Interpolation}(\boldsymbol{\mathscr{A}}) \in\mathbb{R}^{H^{\prime}\times W^{\prime}\times \ngs},
    \label{interpolation}
\end{equation}
where $H^{\prime}$ and $W^{\prime}$ are the spatial sizes of input images. Each channel of $\hat{\boldsymbol{\mathscr{A}}}$ corresponds to one group. We use $\hat{\boldsymbol{\mathscr{A}}}$ to interpret our DGML models. Specifically, it can answer what input spatial locations are important for each group and the connections between input spatial locations and output embeddings. 

First, in DGML models, different groups are expected to capture different characteristics of input images. In our A-grouping, we can verify this by visualizing the position-wise attention weight maps in the input space. According to Sec.~\ref{sec_a_grouping}, the $\ngs$ groups of features are computed by $\boldsymbol{F} = \mbox{Softmax}(\boldsymbol{Q}^T
\boldsymbol{K}) \boldsymbol{V}^T$. Note that the value matrix $\boldsymbol{V}$ and key matrix $\boldsymbol{K}$ is shared by all $\ngs$ group while $\ngs$ query vectors are learned for different groups separately. Since the diversity loss encourages different groups to generate different embeddings, the non-shared query vectors are trained to be different. As each query vector is attending to the same key matrix $\boldsymbol{K}$, the generated attention weight maps indicate what characteristics are captured by different queries. By mapping the attention maps to the input image, we obtain $\hat{\boldsymbol{\mathscr{A}}}$ and can use it to understand the meaning of different groups. 
\mycolor{The visualization results reported in Figure~\ref{cub_vis} in Section~\ref{exp_vis} clearly show that different groups focus on different spatial locations on input images.} In addition, a certain group consistently detects the same characteristics of different images. We can observe that group 1 captures the body of birds, group 2 focuses on the head of birds, group 3 focuses on the neck of birds, and group 4 focuses on the background. 

Second, our proposed A-grouping can explain the connections between output features and input images. According to Equation~\eqref{eq10}, the output features are determined by the weight matrix $\boldsymbol{A}$ and the value matrix $\boldsymbol{V}$. By mapping back to $3$-D tensors $\boldsymbol{\mathscr{A}}$ and $\boldsymbol{\mathscr{V}}$, the output feature vector for each group is the summation over the element-wise multiplication between its corresponding weight map and the value tensor. 
We illustrate it in the bottom right part of Figure~\ref{framework}. For the 
$p$-th group, its weight map corresponds to $p$-th channel of the tensor $\boldsymbol{\mathscr{A}}$, denoted as $\boldsymbol{a}_{p}$. Its output feature vector $\boldsymbol{f}_{p}  \in\mathbb{R}^{1 \times D_{V}} $ is a weighted sum of the feature vectors in different spatial locations of $\boldsymbol{\mathscr{V}}  \in\mathbb{R}^{H \times W \times D_{V}}$ and the weights are determined by $\boldsymbol{a}_{p}\in\mathbb{R}^{H \times W}$. Therefore, the elements in $\boldsymbol{a}_{p}$ determines which spatial locations are more important to group $p$ and its corresponding output embeddings. To be concrete, for the spatial location $(i, j)$, its corresponding weight is $a_{p,ij}$, which determines the contribution of $\boldsymbol{v}_{ij} \in \mathbb{R}^{1 \times \feadim_{V}} $ to output feature vector $\boldsymbol{f}_{p}$. For the group $p$, the weight map $\boldsymbol{a}_{p}$ is obtained by attending its corresponding query $\boldsymbol{q}_p$ to the value $\boldsymbol{\mathscr{V}}$. Since the metric loss encourages the output features to capture discriminative input information, then the query is trained to capture discriminative spatial locations such that the weights for these spatial locations are dominant enough. \mycolor{As shown in Figure~\ref{cub_vis} in Section~\ref{exp_vis}, different groups consistently capture different input characteristics and these characteristics, such as bird head, bird neck, and bird body, are discriminative.} This is consistent with our expectation that different groups should capture important but different input information to generate output features.


In comparison, the M-grouping method is not interpretable. It learns groups of features through multiple linear projections. However, it cannot explain what characteristics are detected by different groups. It is unknown which input characteristics are important to the output features. Meanwhile, the C-grouping method has limited interpretation ability since it learns one weight for each feature map. Hence, it can explain which channel is more important for output embeddings. However, it cannot interpret the connections between input and output features. In addition, the G-grouping method is more interpretable than C-grouping as it learns element-wise weights for the feature maps via CNNs. For each group, the sizes of learned weights are $H \times W \times \feadim_{G}$ where $\feadim_{G}$ is the number of channels. By exploring each $H \times W$ weight map, we can understand the feature-level importance for different spatial locations. However, it is challenging to map back to the input image to study the importance of different spatial locations in the input since there are $\feadim_{G}$ weight maps for each group. In addition, existing studies~\cite{anderson2018bottom} have shown that attention-based methods are more interpretable than CNN methods. Furthermore, the learnable query vectors in our method can encourage the weight maps to have larger values for important spatial locations, which make the attention maps more interpretable. We compare our proposed A-grouping with G-grouping in Section~\ref{exp_vis} to study the interpretation performance.

\subsection{Permutation and Translation Invariance}\label{sec:invariance}
In addition to the interpretability, our proposed A-grouping module is permutation invariant, which is a promising property for representation learning tasks. 
Intuitively, if two input images contain the same objects but they are not spatially aligned, then it is promising to capture the key objects and output the same embedding vectors. In our proposed A-grouping module, the output feature vectors will remain the same no matter how the input spatial locations are permutated. Hence, our method can capture important input characteristics regardless of their spatial locations.  
In the following, we provide a formal definition of matrix column permutation and then prove the permutation invariance property of our proposed A-grouping module. 

\begin{definition} [Permutation matrices and matrix column permutations]
Given a permutation $\pi$ of $n$ elements, the $n \times n$ permutation matrix can be define as $\boldsymbol{U}_{\pi} = [\boldsymbol{e}_{\pi(1)}, \boldsymbol{e}_{\pi(2)}, \cdots, \boldsymbol{e}_{\pi(n)}]$, where $\boldsymbol{e}_{\pi(i)} \in \mathbb{R}^{n}$ is a one-hot vector whose $\pi(i)$-th element is $1$.
Given a matrix $\boldsymbol{B} \in \mathbb{R}^{m \times n}$ and a permutation $\pi$, a matrix column permutation is a transformation $\mathcal{T}_\pi: \mathbb{R}^{m \times n} \to \mathbb{R}^{m \times n}$, defined as
\begin{equation}
   \mathcal{T}\pi(\boldsymbol{B}) = \boldsymbol{B}\boldsymbol{U}_{\pi}.
\end{equation}
\end{definition}
Here the matrix column permutation $\mathcal{T}\pi(\boldsymbol{B})$ permutes the columns of $\boldsymbol{B}$ using the permutation $\pi$.
Based on this definition, the permutation invariance of our method can be proved in the following theorem.

\begin{theorem} [Permutation invariance]
Given the query vector $\boldsymbol{q} \in \mathbb{R}^{D_{K}}$ corresponding to a particular group, the key matrix $\boldsymbol{K} \in \mathbb{R}^{D_{K} \times HW}$, the value matrix $\boldsymbol{V} \in \mathbb{R}^{D_{V} \times HW}$, and a matrix column permutation operator $\mathcal{T}_\pi$, the operator $\mathcal{G}(\boldsymbol{q}, \boldsymbol{K}, \boldsymbol{V})= \mbox{Softmax}(\boldsymbol{q}^T \boldsymbol{K}) \boldsymbol{V}^T$ is invariant to column permutations on $\boldsymbol{K}$ and $\boldsymbol{V}$. That is, the below equality holds:
\begin{equation}
    \mathcal{G}(\boldsymbol{q}, \boldsymbol{K}, \boldsymbol{V}) = \mathcal{G}\left(\boldsymbol{q}, \mathcal{T}_\pi(\boldsymbol{K}), \mathcal{T}_\pi(\boldsymbol{V})\right).
\end{equation}
\end{theorem}

\begin{proof}
\begin{equation}
\begin{split}
    \mathcal{G}(\boldsymbol{q}, \mathcal{T}_\pi(\boldsymbol{K}), \mathcal{T}_\pi(\boldsymbol{V})) 
    &=\mbox{Softmax}(\boldsymbol{q}^T \mathcal{T}_\pi(\boldsymbol{K})) \mathcal{T}_\pi(\boldsymbol{V})^T  \\
    &= \mbox{Softmax}(\boldsymbol{q}^T \boldsymbol{K} \boldsymbol{U}_{\pi}) (\boldsymbol{V} \boldsymbol{U}_{\pi})^T\\
     &= \frac{1}{\mbox{sum}(\mbox{exp}(\boldsymbol{q}^T \boldsymbol{K} \boldsymbol{U}_{\mycolor{\pi}}))}
    \boldsymbol{q}^T \boldsymbol{K} \boldsymbol{U}_{\pi}
    \boldsymbol{U}_{\pi}^{T}\boldsymbol{V}^{T}\\
    &=\frac{1}{\mbox{sum}(\mbox{exp}(\boldsymbol{q}^T \boldsymbol{K}))}
    \boldsymbol{q}^T \boldsymbol{K} \boldsymbol{V}^{T}\\
    &=\mathcal{G}(\boldsymbol{q}, \boldsymbol{K}, \boldsymbol{V}),
\end{split}
\label{p1}
\end{equation}
where $\mbox{sum}(\cdot)$ is the summation and $\mbox{exp}(\cdot)$ is the element-wise exponential function. Note that $\mbox{sum}(\mbox{exp}(\boldsymbol{q}^T \boldsymbol{K} \boldsymbol{U}_{\pi})) = \mbox{sum}(\mbox{exp}(\boldsymbol{q}^T \boldsymbol{K}))$. This is because that the permutation only change the positions but not the numbers, and these two results are equal after summation over all the positions.
\end{proof}

According to Sec.~\ref{sec_a_grouping}, the key matrix $\boldsymbol{K}$ and the value matrix $\boldsymbol{V}$ is obtained by performing $1 \times 1$ convolution and unfolding operations on the CNN output $\boldsymbol{\mathscr{I}}$. Hence, performing matrix column permutation on $\boldsymbol{K}$ and $\boldsymbol{V}$ simultaneously is equivalent to performing spatial permutation on $\boldsymbol{\mathscr{I}}$. Therefore, our proposed A-grouping module is permutation invariant. In addition, CNNs are translation invariant and the whole framework consists of CNNs and our proposed A-grouping modules. Hence, the whole framework is translation invariant. This property enables our method to capture important input characteristics regardless of their spatial locations, resulting in better robustness and more accurate interpretations. In Figure \ref{shift cars}, we provide several examples to demonstrate the translation invariance property of our framework. It shows that our method precisely captures car headlights and car wheels regardless of their spatial locations. 
\begin{figure}[t]
	\centering
	\includegraphics[width=.98\linewidth]{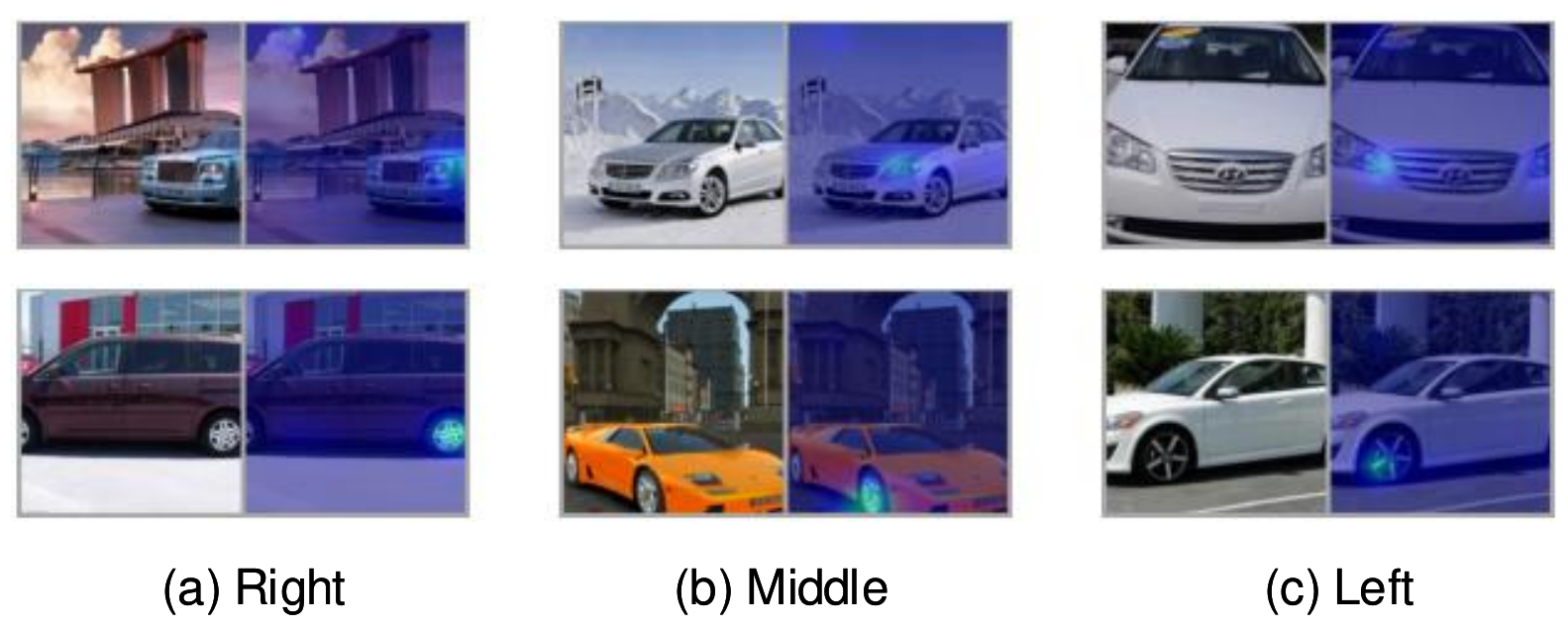}
	\caption{Examples showing the translation invariance property of our method. Each image pair contains the raw image (left) and the attention weight map visualization (right). Our method can capture the key object components in different locations.}
	\label{shift cars}
\end{figure}

\section{Experimental Studies}
In this section, we demonstrate the effectiveness of our proposed method in four aspects:
\begin{itemize}
	\item Our method outperforms the comparing baselines and achieves the state-of-the-art performance, as shown in Sec. \ref{exp_bl}. 
	\item With different backbones and loss functions, our proposed method consistently outperforms other grouping modules,  as discussed in Sec. \ref{exp_grp}.
	\item In Sec. \ref{exp_vis} we demonstrate the interpretability of our proposed A-grouping module. 
\end{itemize}
 
\renewcommand\arraystretch{1.2}
\begin{table*}[htp]
	\centering
	\caption{Three loss functions used in the experiments and their corresponding hyperparameters. }
	\begin{tabular}{l|c|c}
		\hline
		Loss & Formula & Hyperparameters\\
		\hline
		Contrastive &$l_{i,j}*d_{i,j} + (1-l_{i,j})*\left[m-d_{i,j}\right]_+$ & $m = 1$ 
		\\
		\hline
		\multirow{2}{*}{Binomial deviance}
		&$\log\left(1+e^{-\alpha(s_{i,j}-m)\beta_{1}}\right) + $ 
		&\multirow{2}{*}{$m=0.5, \alpha=2, \beta_{1}=25, \beta_{0}=1$}
		\\
		&$(1-l_{i,j}) * \log\left(1+e^{\alpha(s_{i,j}-m)\beta_{0}}\right)$&\\
		\hline
		\multirow{2}{*}{Margin} 
		&$l_{i,j}  * \left[d_{i,j}-(\eta-m)\right]_{+} + $ 
		&\multirow{2}{*}{\mycolor{$m=0.2, \eta=1.2, lr_\eta = 5e-4$}}
		\\
		&$(1-l_{i,j}) * \left[(\eta + m)-d_{i,j}\right]_{+}$&\\
		\hline
	\end{tabular}
	\label{loss}
\end{table*}

\subsection{Experimental Setup}~\label{sec:exp_setup}
We evaluate our model on two computer vision tasks, image retrieval, and image clustering. 
The Pytorch \cite{paszke2017automatic} deep learning framework is used for implementation and the ADAM optimizer \cite{kingma2014adam} is adopted to train the model. For each iteration, we set the batch size to $112$, composed by $56$ classes and $2$ examples for each class. The total embedding size is fixed to $512$, and the embedding size of each group is the quotient of the total embedding size and the number of groups.
In the following, we introduce the datasets, loss functions, and evaluation metrics.

\textbf{Datasets.}
We use three large scale image datasets are used in our experiment, including CUB-200-2011 \cite{wah2011caltech}, Cars-196 \cite{krause20133d}, Stanford Online Products (SOP) \cite{oh2016deep}.
\begin{itemize}
	\item CUB-200-2011~\cite{wah2011caltech} contains $11,788$ images from $200$ bird categories. The training set is composed of $5,864$ images from the first $100$ categories and the testing set is composed of the remaining $5,924$ images from the other $100$ categories. In this dataset, samples are evenly distributed across different categories.
	\item Cars-196~\cite{krause20133d} contains $16,185$ images from $196$ car categories with even distribution. We use the first $98$ categories ($8,054$ images) for training and the remaining $98$ categories ($8131$ images) for testing.
	
	\item SOP~\cite{oh2016deep} contains $120,053$ product images covering $22,634$ categories. We use $59,551$ images of the first $11,318$ categories for training and $60,502$ images of the remaining $11,316$ categories for testing. Note that in the SOP dataset, the data distribution is unbalanced since most classes only contain very few examples. 
\end{itemize}

For data preprocessing, we follow the protocol used in the previous work \cite{roth2020revisiting, milbich2020diva}. Training examples are processed by randomly cropping the raw images and then resize them to $224 \times 224$. Additionally, the random horizontal flipping with a probability of $0.5$ is applied. The evaluation examples are obtained by center cropping and resizing the raw images.

\textbf{Loss Functions.}
In our experiment, the contrastive loss \cite{chopra2005learning}, the binomial deviance loss \cite{yi2014deep}, and the margin loss \cite{wu2017sampling} are use as the metric losses $\mathcal{L}_{\text{met}}$. We summarize their mathematical formulations and hyper-parameters in Table~\ref{loss} \mycolor{ where $lr_\eta$ is the learning rate of $\eta$.}
In addition, the binomial deviance loss is employed as the diversity loss to encourage the diversity among different groups. Finally, the whole loss function can be written as
\begin{equation}
\begin{split}
	\mathcal{L} = \frac{1}{N_b}\sum_{n=1}^{N_b} \mathcal{L}_{\text{met}} + 
		\lambda_{1}* \frac{1}{N_g}\sum_{n=1}^{N_g} \mathcal{L}_{\text{div}} +
		\lambda_2 * \|W\|^{2}_{2},
	\label{overall_loss}
\end{split}
\end{equation}
where $N_b$ and $N_g$ are the numbers of sample pairs and group pairs within one batch. The $L2$ regularization term $\|W\|^{2}_{2}$ is employed to avoid the overfitting problem. \mycolor{Note that all grouping methods are compared fairly with the same metric loss function.}

\textbf{Evaluation Metrics.}
Following the existing study~\cite{roth2020revisiting}, we employ several
evaluation metrics in our experiments, including Recall at $1$,  Recall at $2$, Normalized Mutual Information (NMI) \cite{jegou2010product}, F1 score \cite{sohn2016improved}, and class-wise mean average precision measured on recall (mAP) \cite{roth2020revisiting}.

\renewcommand\arraystretch{1.5}
\begin{table*}[t]
	\centering
	\caption{\mycolor{Comparisons with the state-of-the-art methods on the CUB-200-2011, Cars-196, and  Stanford Online Products (SOP) datasets. Best results are shown in bold.}}
	\begin{tabular}{l|c|ccc|ccc|ccc}
		\hline
		\multicolumn{2}{c|}{Datasets$\rightarrow$} & \multicolumn{3}{c}{CUB-200-2011}& \multicolumn{3}{c}{Cars-196}&\multicolumn{3}{c}{SOP}\\
		\hline
		Method$\downarrow$ & Arch$\downarrow$& R$@$1 & R$@$2 &NMI  & R$@$1  & R$@$2&NMI& R$@$1 & R$@$10 &NMI\\
		\hline

		HDC&I-v1&53.6&65.7&-&73.7&83.2&-&69.5&84.4&-\\
		BIER&I-v1&55.3&67.2&-&78.0&85.8&-&72.7&86.5&-\\
		A-BIER&I-v1&57.5&68.7&-&82.0&89.0&-&74.2&86.9&-\\

		HAML &I-v1&55.2&68.7&65.1&81.1&88.8&\textbf{71.9}&70.7&85.0&\textbf{91.1}\\

		Ours-I & I-v1 &\textbf{62.7}&\textbf{73.8}&\textbf{67.0}&\textbf{82.8}&\textbf{89.6}&69.6&\textbf{75.5}&\textbf{87.8}&90.1\\
		\hline
		\multicolumn{2}{l|}{}& R$@$1 & R$@$2 &NMI  & R$@$1  & R$@$2&NMI& R$@$1 & R$@$2 &NMI\\
		\hline
		Triplet & R-50 &62.9&74.3&67.53&79.1&86.7&65.9&77.4&82.0&90.1\\
		Angular & R-50 &62.1&73.7&67.6&78.0&86.0&66.5&73.2&78.1&89.5\\
		Npair & R-50 &61.0&72.7&66.9&76.1&84.6&66.1&75.9&80.7&89.8\\
		Margin & R-50 &63.1&74.4&68.2&79.9&87.5&67.4&78.4&82.8&90.4\\
		R-Margin & R-50 &64.9&75.6&68.4&82.4&89.1&68.7&78.5&83.0&90.3\\
		MS &R-50&62.8&74.4&68.6&81.7&88.9&69.4&78.0&82.6&90.0\\
		DCESML &R-50&65.9&76.6&69.6&84.6&90.7&70.3&75.9&88.4&90.2\\
		DiVA &R-50&69.2&79.3&71.4&87.6&92.9&72.2&\textbf{79.6}&91.2&90.6\\
		DiVA*&R-50&68.6&79.1&70.8&86.9&92.1&72.3&77.6&90.1&90.0\\		
		\hline
		Ours-R & R-50 &\textbf{70.0}&\textbf{79.8}&\textbf{71.7}&\textbf{88.7}&\textbf{93.2}&\textbf{72.7}&79.2&\textbf{91.8}&\textbf{90.8}\\
		\hline
	\end{tabular}
	\label{exp_baseline}
\end{table*}

\renewcommand\arraystretch{1.5}
\begin{table*}[t]
	\centering
	\caption{\mycolor{Comparisons between our A-grouping method, the M-grouping, G-grouping, and C-grouping methods on the CUB-200-2011 and Cars-196 datasets. All modules are evaluated using two CNN backbones and three metric loss functions. Note that N-grouping denotes the module with no grouping technique. For grouping models, the output embedding sizes are $4 \times 128$, indicating $4$ groups and $128$ feature dimensions for each group. For none grouping models, the embedding type is $1 \times 512$. The best results are shown in bold.}}
	\begin{tabular}{l|l|ccccc|ccccc}
		\hline
		&& \multicolumn{5}{c}{CUB-200-2011}& \multicolumn{5}{c}{Cars-196} \\
		\hline
		&& R$@$1 &R$@$2 &NMI &F1 &mAP &R$@$1 & R$@$2 &NMI &F1 &mAP\\
		\hline
		\multirow{3}[6]*{\tabincell{c}{Inception-v1\\+\\Contrastive}} 
		&N-grouping\hspace{4.5pt}(Emb: $1\times 512$)&54.74&66.96&62.23&30.28&18.55&65.54&76.50&57.77&26.88&16.56\\
		\cline{2-12}
		&M-grouping\hspace{3.7pt}(Emb: $4\times128$)&56.50&\textbf{68.92}&62.90&31.46&19.12&68.10&77.73&59.34&27.98&17.70\\
		&G-grouping\hspace{5.1pt}(Emb: $4\times128$)&54.71&66.80&61.94&30.13&18.61&67.36&77.15&58.83&27.68&18.33\\
		&C-grouping\hspace{5.5pt}(Emb: $4\times128$)&56.45&68.08&63.13&\textbf{31.98}&\textbf{19.83}&\textbf{71.84}&\textbf{81.21}&\textbf{61.41}&\textbf{30.71}&\textbf{19.54}\\
		&A-grouping\hspace{5.0pt}(Emb: $4\times128$)&\textbf{57.06}&68.67&\textbf{63.06}&31.81&19.69&71.34&80.91&60.78&29.31&19.06\\
		
		\hline
		\hline
		\multirow{3}[6]*{\tabincell{c}{Inception-v1\\+\\Binomial}} 
		&N-grouping\hspace{4.5pt}(Emb: $1\times 512$)&57.12&68.30&62.50&31.60&19.23&71.70&80.89&60.85&30.38&17.35\\
		\cline{2-12}
		&M-grouping\hspace{3.7pt}(Emb: $4\times128$)&59.64&71.13&64.45&34.34&20.97&71.05&80.38&60.15&29.50&17.37\\
		&G-grouping\hspace{5.1pt}(Emb: $4\times128$)&57.11&69.01&64.65&\textbf{34.59}&19.26&72.49&81.53&61.55&30.53&18.43\\
		&C-grouping\hspace{5.5pt}(Emb: $4\times128$)&58.42&70.56&64.43&34.24&20.83&74.16&83.18&62.06&31.05&19.11\\
		&A-grouping\hspace{5.0pt}(Emb: $4\times128$)&\textbf{61.60}&\textbf{72.74}&\textbf{64.71}&34.53&\textbf{21.92}&
		\textbf{80.00}&\textbf{87.27}&\textbf{65.27}&\textbf{34.77}&\textbf{22.88}\\
		
		\hline			
		\hline
		\multirow{3}[6]*{\tabincell{c}{Inception-v1\\+\\Margin}} 
		&N-grouping\hspace{4.5pt}(Emb: $1\times 512$)&58.02&68.96&63.47&32.88&19.61&73.12&81.53&60.22&29.41&18.78\\
		\cline{2-12}
		&M-grouping\hspace{3.7pt}(Emb: $4\times128$)&57.77&69.50&64.86&35.00&19.81&72.75&82.01&61.49&31.28&18.50\\
		&G-grouping\hspace{5.1pt}(Emb: $4\times128$)&56.25&68.82&63.15&31.69&18.35&74.36&83.26&62.59&31.94&18.65\\
		&C-grouping\hspace{5.5pt}(Emb: $4\times128$)&60.20&71.66&65.89&36.49&21.94&76.77&84.68&64.22&33.79&19.98\\
		&A-grouping\hspace{5.0pt}(Emb: $4\times128$)&\textbf{62.90}&\textbf{73.13}&\textbf{66.91}&\textbf{37.19}&\textbf{22.96}&
		\textbf{80.23}&\textbf{86.32}&\textbf{65.54}&\textbf{35.28}&\textbf{21.19}\\
		
		\hline			
		\hline
		\multirow{3}[6]*{\tabincell{c}{ResNet-50\\+\\contrastive}} 
		&N-grouping\hspace{4.5pt}(Emb: $1\times 512$)&61.85&73.18&65.89&34.94&23.26&73.32&81.79&61.60&30.32&20.45\\
		\cline{2-12}
		&M-grouping\hspace{3.7pt}(Emb: $4\times128$)&62.69&73.94&66.06&\textbf{35.15}&\textbf{24.00}&71.97&81.54&61.58&31.00&20.29\\
		&G-grouping\hspace{5.1pt}(Emb: $4\times128$)&61.75&73.23&65.75&34.08&23.08&74.90&83.58&62.37&31.47&21.71\\
		&C-grouping\hspace{5.5pt}(Emb: $4\times128$)&62.24&72.81&64.51&32.42&23.24&75.07&83.41&\textbf{63.30}&\textbf{32.16}&\textbf{22.90}\\
		&A-grouping\hspace{5.0pt}(Emb: $4\times128$)&\textbf{63.22}&\textbf{74.12}&\textbf{66.51}&34.59&23.33&
		\textbf{76.07}&\textbf{83.91}&62.24&31.02&22.41\\

		\hline
		\hline	
		\multirow{3}[6]*{\tabincell{c}{ResNet-50\\+\\Binomial}} 
		&N-grouping\hspace{4.5pt}(Emb: $1\times 512$)&65.07&75.59&67.05&36.35&24.99&80.90&87.79&65.77&35.89&23.61\\
		\cline{2-12}
		&M-grouping\hspace{3.7pt}(Emb: $4\times128$)&64.47&75.78&68.39&38.00&24.74&82.33&89.02&67.03&36.98&25.84\\
		&G-grouping\hspace{5.1pt}(Emb: $4\times128$)&64.80&75.91&\textbf{69.12}&\textbf{39.37}&24.80&81.52&88.54&66.42&36.47&24.97\\
		&C-grouping\hspace{5.5pt}(Emb: $4\times128$)&66.54&\textbf{76.91}&68.74&39.11&\textbf{26.39}&83.73&89.95&69.11&\textbf{40.00}&27.25\\
		&A-grouping\hspace{5.0pt}(Emb: $4\times128$)&\textbf{66.87}&76.55&67.72&36.87&24.61&
		\textbf{85.70}&\textbf{90.78}&\textbf{69.71}&39.43&\textbf{27.82}\\
				
		\hline
		\hline	
		\multirow{3}[6]*{\tabincell{c}{ResNet-50\\+\\Margin}} 
		&N-grouping\hspace{4.5pt}(Emb: $1\times 512$)  &63.09&74.42&66.67&36.04&23.64&79.92&87.53&64.10&33.64&22.67\\
		\cline{2-12}
		&M-grouping\hspace{3.7pt}(Emb: $4\times128$)&65.14&75.93&68.16&37.84&24.57&82.25&88.70&67.78&38.22&24.40\\
		&G-grouping\hspace{5.1pt}(Emb: $4\times128$)&64.70&75.83&68.14&38.18&23.75&81.54&89.26&68.19&39.60&24.45\\
		&C-grouping\hspace{5.5pt}(Emb: $4\times128$)&65.01&76.40&68.30&37.53&24.62&84.93&90.79&70.28&41.11&27.90\\
		&A-grouping\hspace{5.0pt}(Emb: $4\times128$)&\textbf{68.62}&\textbf{79.05}&\textbf{70.77}&\textbf{41.31}&\textbf{26.79}&\textbf{87.06}&\textbf{92.24}&\textbf{70.96}&\textbf{41.73}&\textbf{29.15}\\
		
		\hline		
	\end{tabular}
	\label{tab_grouping}
\end{table*}

\subsection{Comparisons with Baselines} \label{exp_bl}
\mycolor{We show the effectiveness of our method by comparing it with several state-of-the-art methods. 
These methods mainly belong to two categories; namely the traditional DML methods with single-embedding outputs and the DGML methods with multi-embedding outputs.
The former category includes   {Triplet} \cite{oh2016deep},   {Angular} \cite{wang2017deep},   {Npair} \cite{sohn2016improved},   {Margin} \cite{wu2017sampling},   {R-Margin} \cite{roth2020revisiting},   {MS} \cite{wang2019multi}, and   {HAML} \cite{zheng2019hardness}. The latter category includes   {HDC} \cite{yuan2017hard},   {BIER} \cite{opitz2018deep},   {A-BIER} \cite{opitz2018deep},   {DCESML} \cite{sanakoyeu2019divide}, and   {DiVA} \cite{milbich2020diva}.
These baseline methods are developed based on two DNN backbones, including the Inception V1 and the ResNet-50. Hence, we implement our methods based on these two backbones, named Ours-I and Ours-R, to compare with these methods accordingly. Specifically, our models 
learn $\ngs = 3$ groups of embeddings, where the embedding size for each group is $170$. To train the model, Ours-I and Ours-R both adopt the margin loss as the metric loss function. The balance weight of the diversity loss $\lambda_{1}$ and regularization of parameters $\lambda_{2}$ in Equation~\eqref{overall_loss} are set to $0.02$ and $0.003$ respectively. The learning rate is
set to $2e-5$ for the models with Inception-V1 backbone and $1e-5$ for the models with ResNet-50 backbone. Finally, we compare different methods using the Recall and NMI metrics.}

\mycolor{The results are reported in Table~\ref{exp_baseline}. Note that the results of all baseline approaches are taken from existing studies~\cite{yuan2017hard, opitz2018deep, zheng2019hardness, milbich2020diva, roth2020revisiting}. Additionally, since we observe DiVA achieves the most competitive results, we reimplement its method  using the same environment as our methods, which is denoted as DiVA$^{*}$.
Obviously, we can observe that our proposed method outperforms all the baselines in most cases. Specifically, the Recall$@$1 (R$@$1) performance of Ours-I is $5.2\%$, $0.7\%$, $1.3\%$ higher than the previous best method over three datasets.The results obtained using ResNet-50 as the backbone show that Ours-R consistently and significantly outperforms the baseline approaches, except DiVA, across all the metrics over these three datasets. It is noteworthy that our method not only achieves superior performance but also demonstrates its interpretability, which is discussed in Section~\ref{exp_vis}.}

\subsubsection{Comparisons with DiVA}
\mycolor{We further discuss the performance when comparing with DiVA. As shown in Table~\ref{exp_baseline}, our method can outperform DiVA$^{*}$, the reimplementation of DiVA,  significantly and consistently. When comparing with DiVA, 
our Ours-R outperforms it in most cases but the gaps are reduced. Note that DiVA is a more complex model and introduces at least \textbf{$\textbf{8}$ more hyperparameters} than our method, which limits the practical applications. In addition, the memory cost of DiVA is higher than Ours-R since it builds a large memory queue and employs an extra running-average network to update it on the fly to estimate the negative distribution for the NCE loss. This mechanism leads to around \textbf{$\textbf{30\%}$ higher memory cost} than our method. Furthermore, DiVA is not interpretable while our method 
demonstrates promising interpretation results, as shown in Section~\ref{exp_vis}. Considering these facts, we believe the results demonstrate the effectiveness of our proposed method.}


\subsection{Comparisons with Grouping Modules} \label{exp_grp}
\mycolor{Next, we further compare our proposed A-grouping module with other grouping strategies in detail since grouping can yield significant performance improvements.}
Specifically, we compare our A-grouping module with four different grouping strategies, including No grouping (N-grouping), the Multi-linear grouping module (M-grouping)~\cite{opitz2018deep}, the Gate grouping module (G-grouping) \cite{kim2018attention} and the Channel grouping module (C-grouping) \cite{chen2019hybrid}. \mycolor{The architectures of different grouping modules are shown in Appendix A. For comprehensive studies,  we explore two different backbones: Inception-v1 \cite{szegedy2015going} and ResNet-50 \cite{he2016deep}, and three types of contrastive loss functions: the contrastive loss  \cite{chopra2005learning}, the binomial deviance loss \cite{yi2014deep}, and the margin loss \cite{wu2017sampling}.} These configurations are denoted as \textbf{Inception-v1 + Contrastive}, \textbf{Inception-v1 + Binomial}, \textbf{Inception-v1 + Margin}, \textbf{ResNet-50 + Contrastive}, \textbf{ResNet-50 + Binomial}, and $\textbf{ResNet-50 + Margin}$ respectively.
\mycolor{Note that for each configuration, these grouping modules are implemented under the same settings for fair comparisons.}
For different grouping methods, we study all configurations on CUB-200-2011 and Cars-196 datasets. Note that we only explore the $\textbf{ResNet-50 + Margin}$ configuration for the SOP dataset due to its expensive training time (around 3 days for each grouping method). 

\mycolor{In these experimental studies, we follow the settings from existing work~\cite{kim2018attention} to set the number of groups to $4$ and the embedding size for each group is $128$. The hyper-parameters in Equation~\eqref{overall_loss} are set to  $\lambda_1=0.01$ and $\lambda_2=0.001$. The learning rates are respectively set to $2e-5$ for the Inception-V1 backbone and $1e-5$ for the ResNet-50 backbone. For evaluations, we adopt the Recall, NMI, F1 score, and mAP metrics.} The results are reported in Table~\ref{tab_grouping} and Table~\ref{tab_sop}, from which we have the following observations:

\renewcommand\arraystretch{1.5}
\begin{table*}[t]
	\centering
	\caption{Comparisons between our A-grouping method, the M-grouping, G-grouping, and C-grouping methods on the Stanford Online Products dataset. The ResNet-50 is employed as CNN backbone and the margin loss is applied as the metric loss. The best results are shown in bold.}
	\begin{tabular}{l|l|ccccc}
		\hline
		&& \multicolumn{5}{c}{SOP}\\
		\hline
		&& R$@$1 &R$@$2 &NMI &F1 &mAP \\
		\hline	
		\multirow{3}[6]*{\tabincell{c}{ResNet-50\\+\\Margin}} 
		&N-grouping\hspace{4.5pt}(Emb: $1\times 512$)  &78.32&82.71&90.19&37.37&42.38\\
		\cline{2-7}
		&M-grouping\hspace{3.7pt}(Emb: $4\times128$)&78.60&82.92&90.40&38.54&42.59\\
		&G-grouping\hspace{5.1pt}(Emb: $4\times128$)&76.58&81.23&89.96&36.33&40.34\\
		&C-grouping\hspace{5.5pt}(Emb: $4\times128$)&77.66&82.03&90.18&37.46&41.63\\
		&A-grouping\hspace{5.0pt}(Emb: $4\times128$)&\textbf{78.70}&\textbf{83.11}&\textbf{91.95}&\textbf{38.55}&\textbf{42.67}\\
		
		\hline	
	\end{tabular}
	\label{tab_sop}
\end{table*}

\begin{itemize}
	\item Our proposed A-grouping module outperforms the other three grouping modules in most configurations with different backbones and loss functions.
	\item G-grouping generally obtains inferior performance than C-grouping and A-grouping modules.
	This is because G-grouping learns element-wise weights and may cause the over parameterized issue.
	\item The configuration \textbf{ResNet-50 + Margin} achieves the best performance among all possible combinations.
	\item Our proposed A-grouping module performs better on the CUB-200-2011 and Cars-196 datasets than the SOP dataset. The main reason is the unbalanced data distribution of the SOP dataset that most classes only contain less than 5 images. This is consistent with the observations in the existing studies~\cite{roth2020revisiting, opitz2018deep}.
\end{itemize}
\begin{figure*}[h!]
	\centering
	\includegraphics[width=0.95\linewidth, height=0.75\linewidth]{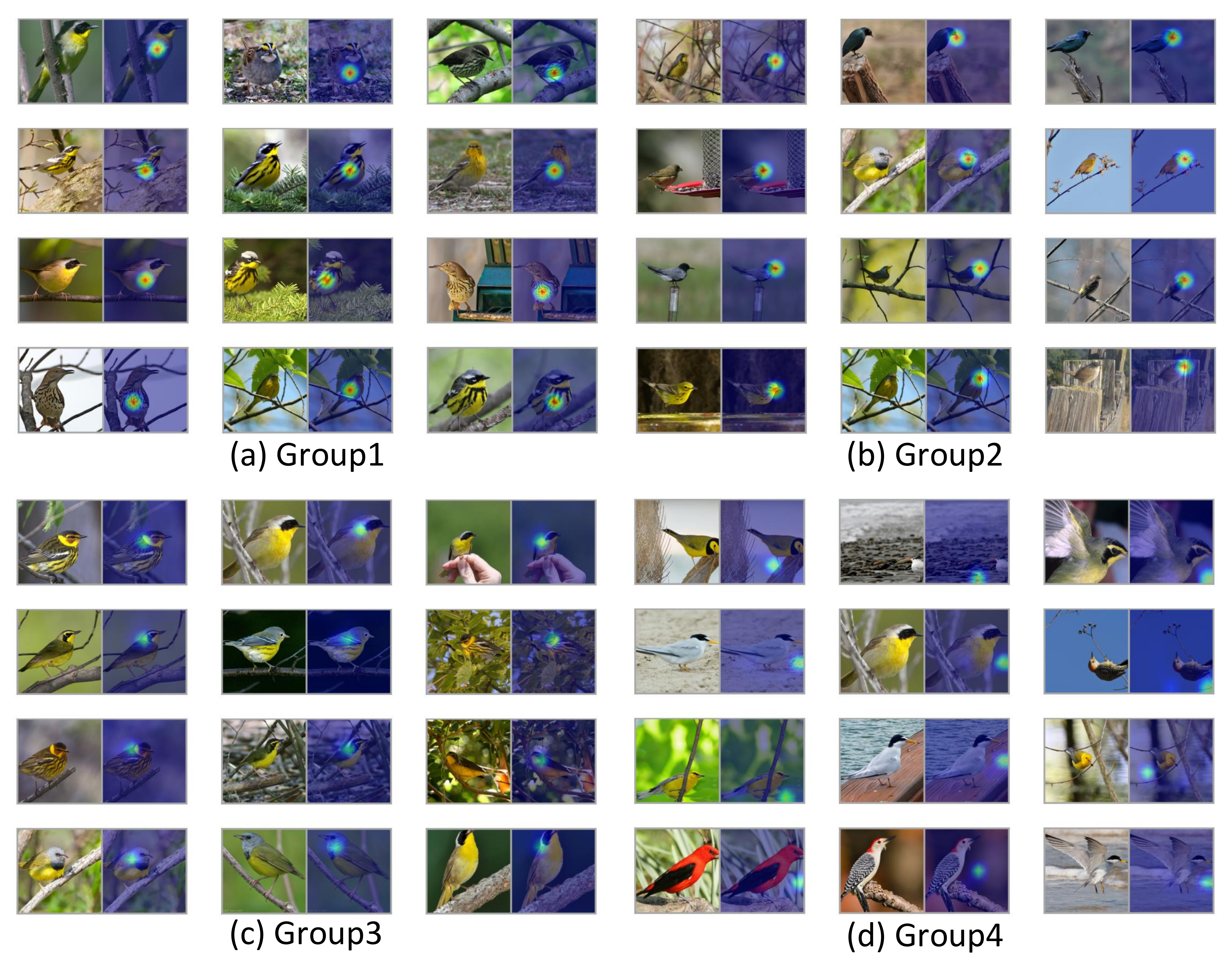}  
	\caption{Visualization results of the CUB-200-2011 dataset. Each image pair composes of the raw image on the left side and the attentive map visualization on the right side. The highlighted regions indicate the spatial locations that different groups are detecting. The four groups in our A-grouping focus on different patterns of the bird images regardless of the spatial locations of these patterns. These patterns are (a) the body, (b) the head, (c) the neck, and (d) the background.}
	\label{cub_vis}
\end{figure*}

\subsection{Study of Interpretability} \label{exp_vis}
\mycolor{Finally, we study the interpretability of different grouping methods on the CUB-200-2011 and Cars-196 datasets. We first study the visualization results of our proposed A-grouping based on the  {ResNet-50 + Margin} configuration.}
Specifically, we explore what input information is detected by different groups by visualizing the learned weight maps. Since our method has four groups of attentive weight maps, we resize the weight map and project it to the input image for each group. Hence, the visualization results can indicate which input regions are captured by different groups.

For the CNN backbone ResNet-50, each attention weight map is $7\times7$ since the sizes of ResNet-50 output feature maps are $7\times7\times2048$. Assuming there are $N$ images, we can obtain $4 \times N$ attention maps and each group has $N$ attention maps. Based on the values of these weight maps, we pick $12$ attention maps containing the top attentive weights for each group. The visualization results of the CUB-200-2011 and Cars-196 datasets are shown in Figure~\ref{cub_vis} and Figure~\ref{cars_vis} respectively. From Figure~\ref{cub_vis}, we can observe that the four groups in our A-grouping module focus on the body, head, neck, and background of the birds respectively. In addition, as shown in Figure~\ref{cars_vis}, the groups in our method capture the the characteristics of the glass window, the headlight, the wheel, and the background of cars respectively. Such visualization results clearly explain the meaning of different groups, which demonstrate the interpretability of our proposed method. The learnable queries can be understood as pattern detectors and large weights are generated once a particular pattern is detected. In addition, the experimental results show that our proposed A-grouping is robust to the spatial shifting and the orientation changes. For example, group 2 capture bird head precisely no matter where the bird head pattern is located and how the bird head is oriented. Furthermore, these patterns detected by our groups contain important and discriminative information to identify the image labels. It is shown that the groups in our method can capture important but different input characteristics. 

\mycolor{We also compare the visualization results for G-grouping, C-grouping, and our A-grouping modules, based on the  {ResNet-50 + Margin} configuration.
For C-grouping and G-grouping, we choose the channel with the strongest response for each group, and visualize the corresponding weight maps in the input space.}
 We report the visualization results in Figure~\ref{vis_com}. It is clear that our A-grouping can focus on more fine-grained information than the \mycolor{G-grouping and C-grouping.} Furthermore, different groups in our method can capture different but specific input patterns while all groups of the G-grouping method focus on the main objects of the images. Overall, the visualization results show that our method tends to generate more reasonable features than other grouping methods. 

\begin{figure*}[h!]
	\centering
	\includegraphics[width=0.95\linewidth, height=0.75\linewidth]{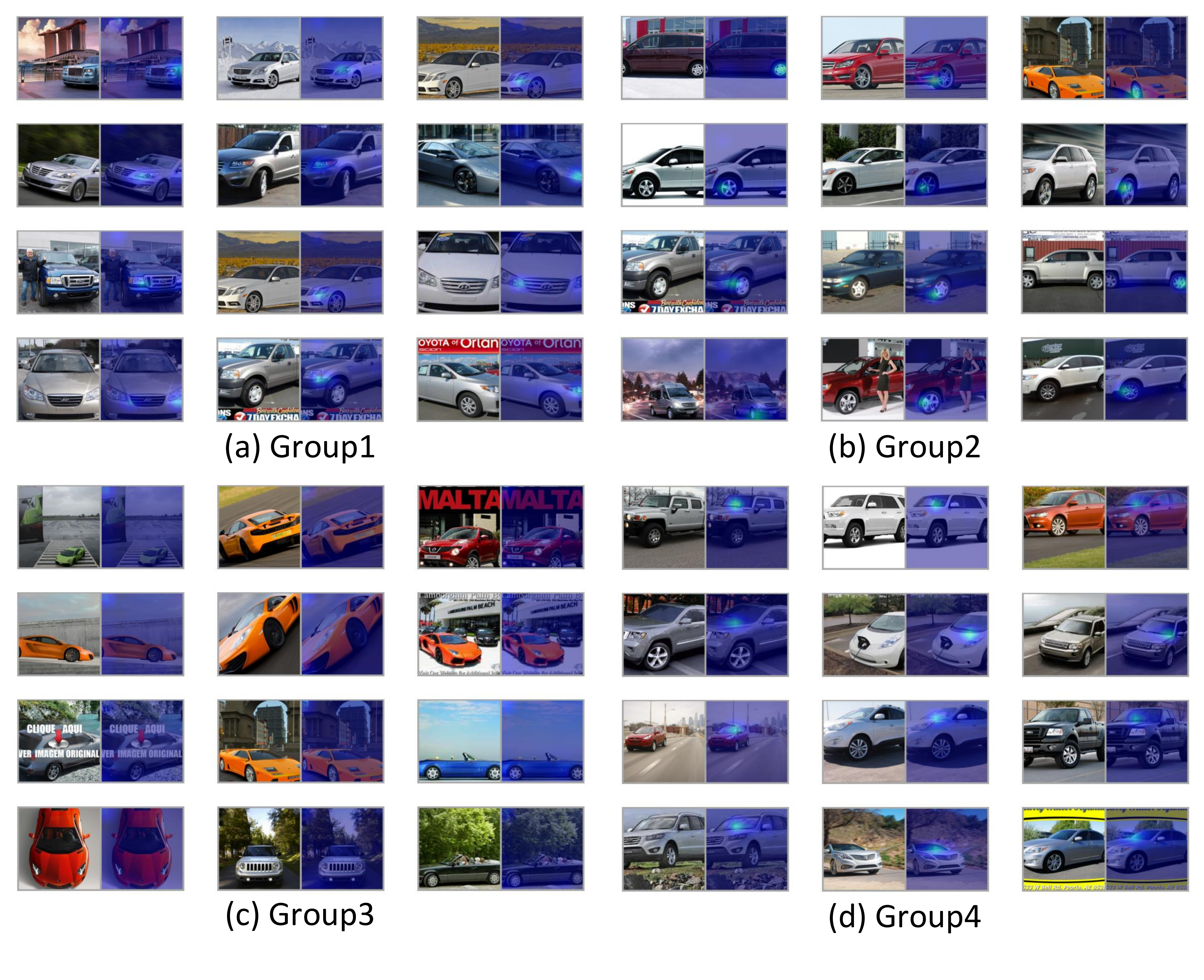}  
	\caption{The visualization results of the Cars-196 dataset. The four groups of the A-grouping module focus on (a) the glass window, (b) the headlight, (c) the wheel, and (d) the background of the cars images respectively.}
	\label{cars_vis}
\end{figure*}
\begin{figure*}[t!]
	\centering
	\includegraphics[width=0.95\linewidth, height=0.3\linewidth]{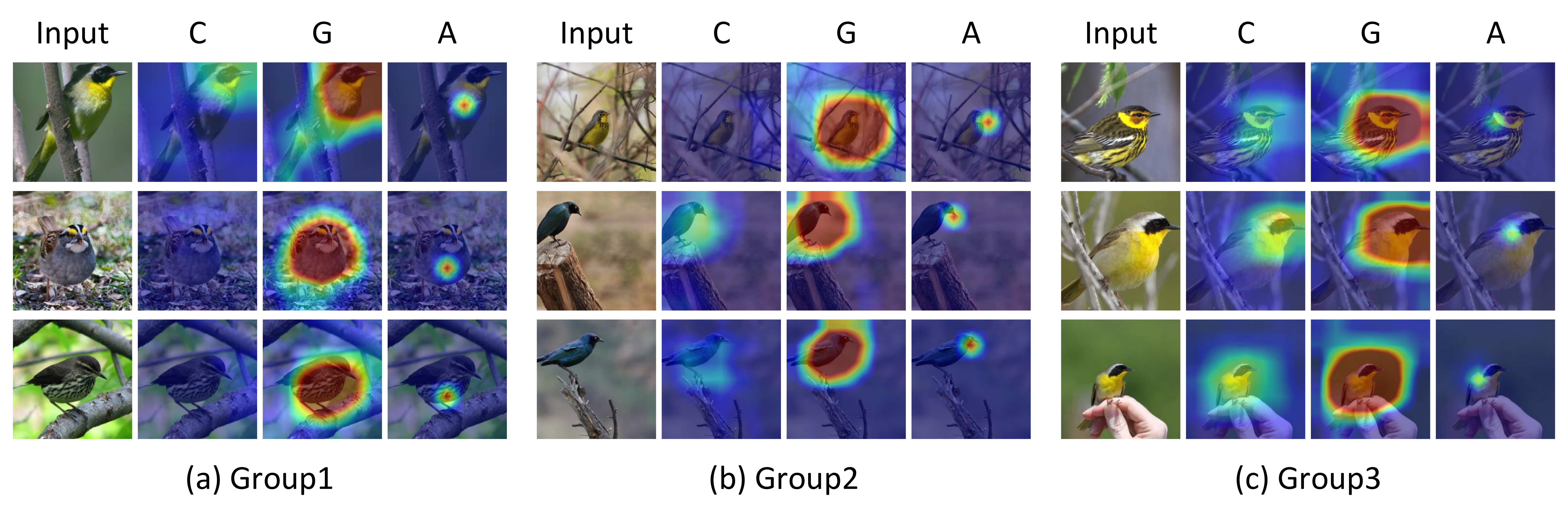}  	
	\caption{\mycolor{The visualization comparisons among C-grouping (C), G-grouping (G), and A-grouping (A). Our A-grouping focuses on more fine-grained information than the other two grouping modules.} In addition, each group of A-grouping are capturing different image patterns while all groups in G-grouping focus on the whole objects of the images.}
	\label{vis_com}
\end{figure*}

\section{Conclusion}
In this work, we study deep metric learning and propose an improved and interpretable grouping method, known as the A-grouping. A-grouping is more powerful than existing grouping methods in computing feature representations. More importantly, our proposed A-grouping is naturally interpretable.
We conduct thorough experiments on image retrieval and clustering tasks to evaluate the effectiveness of our method. We show that A-grouping achieves improved results across different datasets, base models, loss functions, and evaluation metrics. We perform ablation studies on grouping modules and show that our grouping module outperforms existing modules significantly. We also present comprehensive visualization interpretable results.

\bibliography{reference,attn,ji}
\bibliographystyle{IEEEtran}

\end{document}